\documentclass[11pt]{article}
\usepackage[top=1in,bottom=1in,left=1in,right=1in]{geometry} 
\usepackage{palatino}

\usepackage[T1]{fontenc}
\usepackage{lmodern}
\usepackage[american]{babel}
\usepackage{tabularx}
\usepackage{booktabs}
\usepackage[table]{xcolor}

\usepackage{enumerate}
\usepackage{verbatim}
\usepackage{natbib}
\usepackage{amsmath}
\usepackage{subfigure}
\usepackage{enumerate}
\usepackage{subfigure}
\usepackage{verbatim}
\usepackage{natbib}
\usepackage{multirow}
\usepackage{graphicx}
\usepackage{amssymb,amsmath}
\usepackage{color}
\usepackage{hyperref}
\usepackage{url}

\normalsize

\def\be{\begin{eqnarray}}
\def\ee{\end{eqnarray}}
\usepackage{epsfig}

\newtheorem{thm}{Theorem}

\title{Convergence Rates for Neural-Network Estimation with Current-Status Data} 

\author{ Yuan Wu$^\ast$ \and Tianhui Zhou$^\dagger$ } 

\date{} 

\begin{document}

\maketitle 

\begin{center} 
Department of Biostatistics and Bioinformatics, Duke University,\\ Durham, North Carolina, USA\\[2pt] $^\ast$\texttt{yuan.wu@duke.edu}, $^\dagger$\texttt{tianhui.zhou@duke.edu} 
\end{center}

\begin{abstract}
Current-status data arise when an event time is observed only through an indicator of whether it occurred before an examination time. This paper studies a nonparametric neural-network sieve maximum likelihood estimator of the conditional cumulative distribution function of the event time. Under H\"older smoothness assumptions, we establish an explicit convergence rate by combining approximation theory for rectified linear unit neural networks with empirical-process arguments. This result provides theoretical support for neural-network estimation and subsequent inference under current-status observation.
\end{abstract}

\section{Introduction}

Current-status data arise when the event time of interest is not observed exactly. Instead, each subject is examined at a single observation time, and only whether the event occurred before that time is recorded~\cite[]{sun:06}. Such data occur naturally in epidemiological studies, medical screening studies, reliability investigations, and other applications in which continuous monitoring of the event time is unavailable~\cite[]{grummerstrawn:1993}.

Let $Y$ denote an event time, let $Z$ denote an examination time, and let $X$ be a vector of covariates. The observed current-status indicator is
\[
B=I(Y<Z).
\]
The conditional distribution function
\[
g_0(z,x)=P(Y<z\mid X=x)
\]
characterizes the conditional distribution of the event time and is therefore a fundamental estimand in regression analysis with current-status data. Conditional on $(Z,X)$, the indicator $B$ follows a Bernoulli distribution with success probability $g_0(Z,X)$.

Neural networks provide a flexible nonparametric framework for estimating the conditional distribution function, particularly when its dependence on the covariates is nonlinear or high-dimensional. This paper studies a neural-network sieve maximum likelihood estimator and establishes conditions under which it attains a rate suitable for subsequent inference.

A mathematically equivalent binary-learning criterion was used by~\cite{zhou:2021}
in the Collaborating Networks framework. Their
setting, however, assumed completely observed outcomes and introduced
the threshold variable as a device for learning the conditional
distribution; it was not formulated as a current-status survival model.
The present paper identifies the connection with current-status data
and studies the corresponding neural-network sieve maximum likelihood
estimator from a survival-analysis perspective. In particular, we
establish an explicit convergence rate that was not derived in the original Collaborating Networks analysis.

The remainder of the paper is organized as follows. Section~\ref{sec:method} introduces the current-status likelihood and the neural-network sieve estimator. Section~\ref{sec:rate} states and proves the convergence-rate theorem. 

\section{Neural-Network Estimation under Current-Status Observation}
\label{sec:method}

Let
\[
O=(B,Z,X)
\]
denote the observed data, where
\[
B=I(Y<Z),
\]
for event time \(Y\), examination time \(Z\), and covariate vector \(X\). 

Suppose that $O_1,\ldots,O_n$ are independent and identically distributed copies of $O$. Define
\[
W=(Z,X)
\]
and
\[
g_0(w)=g_0(z,x)=P(Y<z\mid X=x).
\]
We assume that the event time $Y$ and examination time $Z$ are
conditionally independent given $X$. It follows that
\[
P(B=1\mid Z=z,X=x)
=
P(Y<z\mid X=x)
=
g_0(z,x).
\]
Consequently, the conditional likelihood of $B$ given $W$ is Bernoulli with success probability $g_0(W)$.

For a candidate conditional distribution function $g$, define the population and empirical log-likelihood criteria by
\[
M(g)
=
E\left[
B\log g(W)
+
(1-B)\log\{1-g(W)\}
\right]
\]
and
\[
M_n(g)
=
\frac{1}{n}
\sum_{i=1}^n
\left[
B_i\log g(W_i)
+
(1-B_i)\log\{1-g(W_i)\}
\right].
\]

For measuring estimation error, define
\[
d^2(g,g_0)
=
E\left[\{g(W)-g_0(W)\}^2\right].
\]

Let $\mathcal H_n$ be a class of bounded rectified linear unit neural (ReLU) networks~\cite[]{nair:2010} with complexity parameter $V_n$ (the number of local approximation pieces used in the ReLU networks) for $V_n<n^A$ and $0<A<1$. In addition, assume the number of depths of the ReLU networks \(D_n\asymp \log V_n\). For a fixed constant $\epsilon>0$, define
\[
g_h(w)
=
\epsilon+(1-2\epsilon)\sigma\{h(w)\},
\qquad
\sigma(u)=\frac{1}{1+e^{-u}},
\]
and let
\[
\mathcal G_n
=
\{g_h:h\in\mathcal H_n\}.
\]
The neural-network sieve maximum likelihood estimator is defined by
\[
\widehat g_n
\in
\arg\max_{g\in\mathcal G_n}M_n(g).
\]

The transformation above ensures that
\[
\epsilon\leq g_h(w)\leq 1-\epsilon
\]
for every $h\in\mathcal H_n$ and every $w$ in the covariate domain. 

\section{Convergence Rate}
\label{sec:rate}

The following theorem establishes the convergence rate of $\widehat g_n$.

\begin{thm}\label{thm}
Assume, without loss of generality, that the covariate vector \(w\) has
support on \([0,1]^d\), obtained through a one-to-one rescaling of a bounded
covariate domain. Suppose there exists \(\epsilon>0\) such that
\[
2\epsilon
<
g_0(w)
<
1-2\epsilon,
\quad \text{with} \quad
w\in[0,1]^d.
\]
Define
\[
h_0(w)
=
\log\left\{
\frac{g_0(w)-\epsilon}
{1-\epsilon-g_0(w)}
\right\}.
\]
Assume that
\[
h_0\in \mathcal H^\beta\left([0,1]^d\right)
\]
for some \(\beta>0\). Then
\[
d\left(\hat g_n,g_0\right)
=
O_p\!\left\{
n^{-\beta/(2\beta+d)}
(\log n)^{4\beta/(2\beta+d)}
\right\}.
\]
In particular,
\[
d\left(\hat g_n,g_0\right)
=
o_p\left(n^{-1/4}\right)
\]
whenever \(\beta>d/2\).
\end{thm}

{\em Proof.}
Write
\[
m_g(O)=B\log g(W)+(1-B)\log\{1-g(W)\}.
\]
By construction,
\[
\epsilon\le g_h(w)\le 1-\epsilon
\]
for all $h\in\mathcal H_n$ and all $w\in[0,1]^d$. Therefore $m_g$ is
uniformly bounded over $g\in\mathcal G_n$. Moreover, since the derivative of
$\log x$ and $\log(1-x)$ is uniformly bounded on $[\epsilon,1-\epsilon]$,
there exists a constant $C_\epsilon<\infty$ such that
\[
|m_g(O)-m_{\tilde g}(O)|
\le C_\epsilon |g(W)-\tilde g(W)|
\]
for all $g,\tilde g\in\mathcal G_n$.

Next,
\[
M(g_0)-M(g)
=
E\left[
g_0(W)\log\frac{g_0(W)}{g(W)}
+
\{1-g_0(W)\}\log\frac{1-g_0(W)}{1-g(W)}
\right].
\]
Since $g_0$ and every $g\in\mathcal G_n$ are uniformly bounded away from
$0$ and $1$, a Taylor expansion of the Bernoulli Kullback--Leibler divergence
implies that there exist constants $0<c<C<\infty$ such that
\[
c\,d^2(g,g_0)
\le
M(g_0)-M(g)
\le
C\,d^2(g,g_0)
\]
for all $g\in\mathcal G_n$.

By the ReLU approximation theorem for H\"older functions ~\cite[Theorem 5,][]{schmidthieber2020}, together with $D_n \asymp \log V_n$, there exists
$h_n^*\in\mathcal H_n$ such that
\[
\|h_n^*-h_0\|_\infty
\lesssim V_n^{-\beta/d}.
\]
Since the map
\[
h\mapsto \epsilon+(1-2\epsilon)\sigma(h)
\]
is Lipschitz, the corresponding function $g_n^*=g_{h_n^*}\in\mathcal G_n$
satisfies
\[
d(g_n^*,g_0)\le a_n,
\quad \text{with} \quad
a_n\asymp V_n^{-\beta/d}.
\]
Hence
\[
M(g_0)-M(g_n^*)\le C a_n^2 .
\]

We now establish the localized empirical-process bound used below. For
$r>0$, define
\[
\mathcal F_n(r)
=
\left\{
m_g-m_{g_n^*}:
g\in\mathcal G_n,\ d(g,g_0)\le r
\right\}.
\]
For any $f=m_g-m_{g_n^*}\in\mathcal F_n(r)$,
\[
\|f\|_{L_2(P)}
\le C_\epsilon d(g,g_n^*)
\le C_\epsilon\{r+a_n\},
\]
that is, the class $\mathcal F_n(r)$ has a uniformly bounded envelope $\sigma_r=C_\epsilon\{r+a_n\}$ in terms of \(L_2(P)\) norm.
Also, by the Lipschitz property of $g\mapsto m_g$, the entropy condition
on $\mathcal G_n$ and Lemma 5 in~\cite{schmidthieber2020},
\[
\log N\{\delta,\mathcal F_n(r),L_2(P)\}
\lesssim V_n\log n \log(n/\delta),
\quad \text{with} \quad 0<\delta<1 .
\]
More specifically, 
By Lemma 5 of~\cite{schmidthieber2020}, together with \(D_n\asymp \log V_n\) and \(V_n<n^A\), the network class
\(\mathcal H_n\) satisfies
\[
\log N\{\delta,\mathcal H_n,\|\cdot\|_\infty\}
\lesssim V_n\log n\log(n/\delta),
\quad \text{with} \quad 0<\delta<1.
\]
Since \(h\mapsto g_h=\epsilon+(1-2\epsilon)\sigma(h)\) is Lipschitz,
\[
\log N\{\delta,\mathcal G_n,\|\cdot\|_\infty\}
\lesssim V_n\log n\log(n/\delta).
\]
Moreover, for any \(g_1,g_2\in\mathcal G_n\),
\[
\|m_{g_1}-m_{g_2}\|_{L_2(P)}
\le
C_\epsilon \|g_1-g_2\|_{L_2(P)}
\le
C_\epsilon \|g_1-g_2\|_\infty .
\]
Therefore an \(L_\infty\)-cover of \(\mathcal G_n\) induces an
\(L_2(P)\)-cover of \(\{m_g:g\in\mathcal G_n\}\). Since translating by
\(m_{g_n^*}\) does not change covering numbers,
\[
\log N\{\delta,\mathcal F_n(r),L_2(P)\}
\lesssim V_n\log n\log(n/\delta),
\quad \text{with} \quad 0<\delta<1.
\]

Then the entropy integral satisfies
\[
J(\sigma_r,\mathcal F_n(r))
=
\int_0^{\sigma_r}
\sqrt{1+\log N\{u,\mathcal F_n(r),L_2(P)\}}\,du
\lesssim
\sigma_r\sqrt{V_n\log^2 n},
\]
where logarithmic constants are absorbed into $\log n$.

By the standard bounded maximal inequality for empirical processes~\cite[Theorem 3.4.2,][]{van:wellner:96},
\[
E\sup_{f\in\mathcal F_n(r)}
\left|
\frac1n\sum_{i=1}^n \{f(O_i)-Ef(O_i)\}
\right|
\lesssim
\frac{J(\sigma_r,\mathcal F_n(r))}{\sqrt{n}}
+\frac{J^2(\sigma_r,\mathcal F_n(r))}{\sigma_r n}
\lesssim
(r+a_n)\left(\sqrt{\frac{V_n\log^2 n}{n}}
+
\frac{V_n\log^2 n}{n}\right).
\]
Equivalently, for each deterministic $r>0$,
\[
E
\sup_{d(g,g_0)\le r}
\left|
\{M_n-M\}(g)-\{M_n-M\}(g_n^*)
\right|
\lesssim
(r+a_n)\left(\sqrt{\frac{V_n\log^2 n}{n}}
+
\frac{V_n\log^2 n}{n}\right).
\]

We now apply this bound to the random estimator $\hat g_n$ by a standard
peeling argument; see Section 3.4.1 of~\cite{van:wellner:96}.
Let
\[
b_n=\sqrt{\frac{V_n\log^2 n}{n}},
\quad \text{and} \quad
\rho_n=a_n+b_n .
\]
Partition the parameter space into dyadic shells
\[
\mathcal S_j
=
\left\{
g\in\mathcal G_n:
2^{j-1}\rho_n<d(g,g_0)\le 2^j\rho_n
\right\},
\quad \text{for} \quad j=0,1,2,\ldots .
\]
For each shell, define
\[
Z_j=
\sup_{g\in\mathcal S_j}
\left|
\{M_n-M\}(g)-\{M_n-M\}(g_n^*)
\right|.
\]
Since
\[
\mathcal S_j\subset
\{g\in\mathcal G_n:d(g,g_0)\le 2^j\rho_n\},
\]
the preceding expectation bound with $r=2^j\rho_n$ gives
\[
E Z_j
\lesssim
(2^j\rho_n+a_n)(b_n+b_n^2) .
\]
Because \(V_n=n^A\) with $0<A<1$, we have $b_n<1$ (indeed, $b_n\to0$) as $n\to\infty$.  In addition, from $\rho_n=a_n+b_n$, it follows that
\[
E Z_j
\lesssim
2^j\rho_n b_n .
\]

Since $0\le g, g_0\le 1$, we have
\[
d(g,g_0)\le 1
\]
for all $g\in\mathcal G_n$. Therefore it is sufficient to consider
\[
J_n=
\left\lceil
\log_2(1/\rho_n)
\right\rceil
\]
dyadic shells.

For a fixed constant $K>0$, define
\[
A_j
=
\left\{
Z_j>
KJ_n\,2^j\rho_n b_n
\right\}.
\]

By Markov's inequality,
\[
P(A_j)
\le
\frac{EZ_j}
     {KJ_n\,2^j\rho_n b_n}
\lesssim
\frac1{KJ_n}.
\]

Therefore, by the union bound,
\[
P\left(
\bigcup_{j=0}^{J_n}A_j
\right)
\le
\sum_{j=0}^{J_n}P(A_j)
\lesssim
\frac1K .
\]

Hence, with probability at least $1-C/K$,
\[
Z_j
\le
KJ_n\,2^j\rho_n b_n,
\quad
\text{for all} \quad 0\le j\le J_n .
\]

On this event, let $j$ denote the shell index such that
$\hat g_n\in\mathcal S_j$.

If $j=0$, then
\[
R_n=d(\hat g_n,g_0)\le \rho_n.
\]

If $j\ge1$, then
\[
2^{j-1}\rho_n<R_n\le 2^j\rho_n,
\]
so
\[
2^j\rho_n\le 2R_n.
\]

Therefore, with probability at least \(1-C/K\),
\[
\left|
\{M_n-M\}(\hat g_n)
-
\{M_n-M\}(g_n^*)
\right|
\lesssim
KJ_n(R_n+\rho_n)b_n.
\]

Since
\[
V_n\le n^A
\]
for all sufficiently large $n$ and
\[
a_n=V_n^{-\beta/d},
\]
we have
\[
a_n\ge n^{-A\beta/d}.
\]
Because
\[
\rho_n=a_n+b_n\ge a_n,
\]
it follows that
\[
\rho_n\ge n^{-A\beta/d}.
\]
Therefore
\[
\frac1{\rho_n}\le n^{A\beta/d},
\]
and hence
\[
J_n
=
\left\lceil
\log_2(1/\rho_n)
\right\rceil
=
O(\log n).
\]
Consequently, from
\[
\rho_n=a_n+b_n,
\]
and the fact that \(K\) can be sufficiently large, it follows that
\[
\left|
\{M_n-M\}(\hat g_n)
-
\{M_n-M\}(g_n^*)
\right|
=
O_p\!\left(
\{(R_n+a_n)b_n+b_n^2\}\log n
\right).
\]

Since $\hat g_n$ maximizes $M_n$ over $\mathcal G_n$,
\[
M_n(\hat g_n)\ge M_n(g_n^*).
\]
Therefore,
\[
\begin{aligned}
M(g_0)-M(\hat g_n)
&=
\{M(g_0)-M(g_n^*)\}
+
\{M(g_n^*)-M(\hat g_n)\} \\
&\le
C a_n^2
+
\{M_n-M\}(\hat g_n)
-
\{M_n-M\}(g_n^*).
\end{aligned}
\]

Using the quadratic lower bound for the population criterion and the preceding
empirical-process bound, we obtain
\[
R_n^2
=
O_p\left[
a_n^2
+
\{(R_n+a_n)b_n+b_n^2\}\log n
\right],
\]
where
\[
b_n=\sqrt{\frac{V_n\log^2 n}{n}} .
\]

Let
\[
\tilde b_n=b_n\log n.
\]
Then
\[
R_n^2
=
O_p\left[
a_n^2
+
(R_n+a_n)\tilde b_n
+
\frac{\tilde b_n^2}{\log n}
\right]
=O_p\left[
a_n^2
+
(R_n+a_n)\tilde b_n
+
\tilde b_n^2
\right]
\]

Since
\[
a_n^2+a_n\tilde b_n+\tilde b_n^2
\lesssim
(a_n+\tilde b_n)^2,
\]
we obtain
\[
R_n^2
=
O_p\left[
R_n\tilde b_n
+
(a_n+\tilde b_n)^2
\right].
\]
Solving the quadratic inequality yields
\[
R_n
=
O_p(a_n+\tilde b_n),
\]
where
\[
\tilde b_n
=
\sqrt{\frac{V_n\log^2 n}{n}}\log n .
\]
Hence
\[
d(\hat g_n,g_0)
=
O_p\left(
V_n^{-\beta/d}
+
\sqrt{\frac{V_n\log^2 n}{n}}\log n
\right).
\]

To obtain the optimal choice of the network complexity $V_n$, balance the
approximation and stochastic terms:
\[
V_n^{-\beta/d}
\asymp
\sqrt{\frac{V_n\log^2 n}{n}}\log n .
\]
Squaring both sides gives
\[
V_n^{-2\beta/d}
\asymp
\frac{V_n(\log n)^4}{n}.
\]
Equivalently,
\[
V_n^{1+2\beta/d}
\asymp
\frac{n}{(\log n)^4}.
\]
Therefore,
\[
V_n
\asymp
\left\{
\frac{n}{(\log n)^4}
\right\}^{d/(2\beta+d)}
=
n^{d/(2\beta+d)}
(\log n)^{-4d/(2\beta+d)}.
\]

Substituting this choice of $V_n$ into the preceding rate yields
\[
d(\hat g_n,g_0)
=
O_p\left[
n^{-\beta/(2\beta+d)}
(\log n)^{4\beta/(2\beta+d)}
\right].
\]

Since
\[
\frac{\beta}{2\beta+d}>\frac14
\quad\Longleftrightarrow\quad
\beta>\frac d2,
\]
it follows that
\[
d(\hat g_n,g_0)
=
o_p(n^{-1/4})
\]
whenever $\beta>d/2$, because any finite power of $\log n$ is dominated by
a positive power of $n$.

This completes the proof.

\bibliographystyle{natbib}
\bibliography{paper_ref}

@article{grummerstrawn:1993,
  author  = {Grummer-Strawn, Laurence M.},
  title   = {Regression Analysis of Current-Status Data:
             An Application to Breastfeeding},
  journal = {Journal of the American Statistical Association},
  year    = {1993},
  volume  = {88},
  number  = {423},
  pages   = {758--765},
  doi     = {10.1080/01621459.1993.10476336}
}

@inproceedings{nair:2010, 
author = {Nair, Vinod and Hinton, Geoffrey E.}, 
title = {Rectified Linear Units Improve Restricted Boltzmann Machines}, 
booktitle = {Proceedings of the 27th International Conference on Machine Learning}, 
year = {2010}, 
pages = {807--814}, 
publisher = {Omnipress} 
}

@article{schmidthieber2020,
author = {Schmidt-Hieber, Johannes},
title = {Nonparametric Regression Using Deep Neural Networks with ReLU Activation Function},
journal = {The Annals of Statistics},
volume = {48},
number = {4},
pages = {1875--1897},
year = {2020}
}

@book{sun:06,
     Author = { Sun, Jianguo },
     Title = { The Statistical Analysis of Interval-censored Failure Time Data},
     Publisher = {Springer-Verlag},
     Year = 2006,
     Address = {New York}
     }

@book{van:wellner:96,
     Author = { van der Vaart, A. W. and Wellner, J. A. },
     Title = { Weak Convergence and Empirical Processes},
     Publisher = {Springer},
     Year = 1996,
     Address = {New York}
     }

@article{zhou:2021,
author = {Zhou, Tianhui and Li, Zhiguo and Wu, Yuan and Carlson, David},
title = {Collaborating Networks for Survival Analysis},
journal = {Journal of Machine Learning Research},
volume = {22},
number = {89},
pages = {1--54},
year = {2021}
}

\end{document}